\documentclass[10pt,twocolumn]{article}
\usepackage[
left=2.00cm,
right=2.00cm,
top=2.00cm,
bottom=2.00cm
]{geometry}
\usepackage{titlesec}
\titleformat{\section}[hang]
{\large\bfseries}
{\thesection.}{0.5em}{}

\titleformat{\subsection}[hang]
{\bfseries}
{\thesubsection.}{0.5em}{}

\usepackage{amsmath, amssymb, amsthm}
\usepackage{color, soul}
\usepackage{algorithm, algorithmic, multicol}
\usepackage[titletoc,title]{appendix}
\usepackage{graphicx}
\usepackage{bbm}

\usepackage[compress]{natbib}
\usepackage[utf8]{inputenc} 
\usepackage[T1]{fontenc}    
\usepackage{hyperref}       
\usepackage{url}            
\usepackage{booktabs}       
\usepackage{amsfonts}       
\usepackage{nicefrac}       
\usepackage{microtype}      

\DeclareSymbolFont{bbold}{U}{bbold}{m}{n}
\DeclareSymbolFontAlphabet{\mathbbold}{bbold}

\newtheorem{theorem}{Theorem}
\newtheorem{definition}[theorem]{Definition}
\newtheorem{lemma}[theorem]{Lemma}
\newtheorem{corollary}[theorem]{Corollary}
\newtheorem*{remark}{Remark}

\newcommand\E{\mathbb{E}} 
\newcommand\prob{\mathbb{P}} 
\newcommand\ind{\mathbbold{1}}  
\newcommand\cost{\textbf{c}} 
\newcommand\normcost{\textbf{d}} 
\newcommand\reducedcostseor{\mathcal{C}^{eor}_0} 
\newcommand\unifv{\textbf{u}} 
\newcommand\samplex{\textbf{x}} 
\newcommand\stdv{\textbf{e}} 
\newcommand\cumv{\textbf{s}} 
\newcommand\weightv{\textbf{w}} 
\newcommand\truey{Y} 
\newcommand\predy{\hat{\textbf{y}}} 
\newcommand\potential{\phi} 
\newcommand\probv{\textbf{p}} 
\newcommand\zerov{\textbf{0}} 
\newcommand\rankLoss{L_{\text{rnk}}} 
\newcommand\logLoss{L_{\text{log}}} 
\newcommand\hingeLoss{L_{\text{hinge}}} 
\newcommand\weakpred{\textbf{h}} 
\newcommand\Xv{\textbf{X}} 

\usepackage{footnote}
\makesavenoteenv{tabular}
\makesavenoteenv{table}

\title{Online Boosting Algorithms for Multi-label Ranking}

\author{
  Young Hun Jung\\
  Department of Statistics\\
  University of Michigan\\
  Ann Arbor, MI 48109 \\
  \texttt{yhjung@umich.edu} \\
  \and
  Ambuj Tewari\\
  Department of Statistics\\
  University of Michigan\\
  Ann Arbor, MI 48109 \\
  \texttt{tewaria@umich.edu} \\
}

\begin{document}

%

%

%
%
%
\twocolumn[
\maketitle]


\begin{abstract}
We consider the multi-label ranking approach to multi-label learning. Boosting is a natural method for multi-label ranking as it aggregates weak predictions through majority votes, which can be directly used as scores to produce a ranking of the labels. We design online boosting algorithms with provable loss bounds for multi-label ranking. We show that our first algorithm is optimal in terms of the number of learners required to attain a desired accuracy, but it requires knowledge of the edge of the weak learners. We also design an adaptive algorithm that does not require this knowledge and is hence more practical. Experimental results on real data sets demonstrate that our algorithms are at least as good as existing batch boosting algorithms. 
\end{abstract}

\section{INTRODUCTION}

\textit{Multi-label learning} has important practical applications (e.g., \cite{schapire2000boostexter}), and its theoretical properties continue to be studied (e.g., \cite{koyejo2015consistent}). In contrast to standard multi-class classifications, multi-label learning problems allow multiple correct answers. In other words, we have a fixed set of basic labels, and the actual label is a \textit{subset} of the basic labels. Since the number of subsets increases exponentially as the number of basic labels grows, thinking of each subset as a different class leads to intractability.

It is quite common in applications for the multi-label learner to output a \emph{ranking} of the labels on a new test instance. For example, the popular MULAN library designed by \cite{mulan} allows the output of multi-label learning to be a multi-label ranker. In this paper, we focus on the multi-label ranking (MLR) setting. That is to say, the learner produces a \textit{score vector} such that a label with a higher score will be ranked above a label with a lower score.  We are particularly interested in \textit{online} MLR settings where the data arrive sequentially. The online framework is designed to handle a large volume of data that accumulates rapidly. In contrast to classical \textit{batch learners}, which observe the entire training set, online learners do not require the storage of a large amount of data in memory and can also adapt to non-stationarity in the data by updating the internal state as new instances arrive. 

\textit{Boosting}, first proposed by \cite{freund1997decision}, aggregates mildly powerful learners into a strong learner. It has been used to produce state-of-the-art results in a wide range of fields (e.g., \cite{korytkowski2016fast} and \cite{zhang2014boosted}). Boosting algorithms take weighted majority votes among weak learners' predictions, and the cumulative votes can be interpreted as a score vector. This feature makes boosting very well suited to MLR problems.

The theory of boosting has emerged in batch binary settings and became arguably complete (cf. \cite{schapire2012boosting}), but its extension to an online setting is relatively new. To our knowledge, \cite{chen2012online} first introduced an online boosting algorithm with theoretical justifications, and \cite{beygelzimer2015optimal} pushed the state-of-the-art in online binary settings further by proposing two online algorithms and proving optimality of one. Recent work by \cite{jung2017online} has extended the theory to multi-class settings, but their scope remained limited to single-label problems. 

In this paper, we present the first online MLR boosting algorithms along with their theoretical justifications. Our work is mainly inspired by the online single-label work (\cite{jung2017online}). The main contribution is to allow general forms of weak predictions whereas the previous online boosting algorithms only considered homogeneous prediction formats. By introducing a general way to encode weak predictions, our algorithms can combine binary, single-label, and MLR predictions. 

After introducing the problem setting, we define an \textit{edge} of an online learner over a random learner (Definition \ref{def:onlineWLC}). Under the assumption that every weak learner has a known positive edge, we design an optimal way to combine their predictions (Section \ref{section:optimal}). In order to deal with practical settings where such an assumption is untenable, we present an adaptive algorithm that can aggregate learners with arbitrary edges (Section \ref{section:adaptive}). In Section \ref{section:experiments}, we test our two algorithms on real data sets, and find that their performance is often comparable with, and sometimes better than, that of existing batch boosting algorithms for MLR.

\section{PRELIMINARIES}
The number of candidate labels is fixed to be $k$, which is known to the learner. Without loss of generality, we may write the labels using integers in $[k]:=\{1, \cdots, k\}$. We are allowing multiple correct answers, and the label $\truey_{t}$ is a subset of $[k]$. The labels in $\truey_{t}$ is called \textit{relevant}, and those in  $\truey_{t}^{c}$, \textit{irrelevant}. At time $t=1, \cdots, T$, an \textit{adversary} sequentially chooses a labeled example $(\samplex_{t}, \truey_{t}) \in \mathcal{X} \times 2^{[k]}$, where $\mathcal{X}$ is some domain. Only the instance $\samplex_{t}$ is shown to the learner, and the label $\truey_{t}$ is revealed once the learner makes a prediction $\predy_{t}$. As we are interested in MLR settings, $\predy_{t}$ is a $k$ dimensional score vector. The learner suffers a loss $L^{\truey_{t}}(\predy_{t})$ where the loss function will be specified later in Section \ref{section:optimal}. 

In our boosting framework, we assume that the learner consists of a \textit{booster} and $N$ \textit{weak learners}, where $N$ is fixed before the training starts. This resembles a \textit{manager-worker} framework in that booster distributes tasks by specifying losses, and each learner makes a prediction to minimize the loss. Booster makes the final decision by aggregating weak predictions. Once the true label is revealed, the booster shares this information so that weak learners can update their parameters for the next example.

\subsection{Online Weak Learners and Cost Vector}
We keep the form of weak predictions $\weakpred_{t}$ general in that we only assume it is a distribution over $[k]$. This can in fact represent various types of predictions. For example, a \textit{single-label prediction}, $l \in [k]$, can be encoded as a standard basis vector $\stdv_{l}$, or a \textit{multi-label prediction} $\{l_{1}, \cdots, l_{n}\}$ by $\frac{1}{n}\sum_{i=1}^{n} \stdv_{l_{i}}$. Due to this general format, our boosting algorithm can even combine weak predictions of different formats. This implies that if a researcher has a strong family of binary learners, she can simply boost them without transforming them into multi-class learners through well known techniques such as \textit{one-vs-all} or \textit{one-vs-one} \citep{allwein2000reducing}. 

We extend the \textit{cost matrix framework}, first proposed by \cite{mukherjee2013theory} and then adopted in online settings by \cite{jung2017online}, as a means of communication between booster and weak learners. At round $t$, booster computes a cost vector $\cost^{i}_{t}$ for the $i^{th}$ weak learner $WL^{i}$, whose prediction $\weakpred^{i}_{t}$ suffers the cost $\cost^{i}_{t}\cdot \weakpred^{i}_{t}$. The cost vector is unknown to $WL^{i}$ until it produces $\weakpred^{i}_{t}$, which is usual in online settings. Otherwise, $WL^{i}$ can trivially minimize the cost. 

A binary weak learning condition states a learner can attain over 50\% accuracy however the sample weights are assigned. In our setting, cost vectors play the role of sample weights, and we will define the edge of a learner in similar manner.

Finally, we assume that weak learners can take an importance weight as an input, which is possible for many online algorithms.

\subsection{General Online Boosting Schema}
We introduce a general algorithm schema shared by our algorithms. We denote the weight of $WL^{i}$ at iteration $t$ by $\alpha^{i}_{t}$. We keep track of weighted cumulative votes through $\cumv^{j}_{t} := \sum_{i=1}^{j}\alpha^{i}_{t}\weakpred^{i}_{t}$. That is to say, we can give more credits to well performing learners by setting larger weights. Furthermore, allowing negative weights, we can avoid poor learner's predictions. We call $\cumv^{j}_{t}$ a prediction made by \textit{expert} $j$. In the end, the booster makes the final decision by following one of these experts. 

The schema is summarized in Algorithm \ref{alg:skeleton}. We want to emphasize that the true label $\truey_{t}$ is only available once the final prediction $\predy_{t}$ is made. Computation of weights and cost vectors requires the knowledge of $\truey_{t}$, and thus it happens after the final decision is made. To keep our theory general, the schema does not specify which weak learners to use (line 4 and 12). The specific ways to calculate other variables such as $\alpha^{i}_{t}$, $\cost^{i}_{t}$, and $i_{t}$ depend on algorithms, which will be introduced in the next section.

\begin{algorithm}[h]
	\begin{algorithmic}[1]
		\STATE \textbf{Initialize:} $\alpha^{i}_{1}$ for $i \in [N]$ 
		\FOR {$t = 1, \cdots, T$}
		\STATE Receive example $\samplex_{t}$
		\STATE Gather weak predictions $\weakpred^{i}_{t} = WL^{i}(\samplex_{t}),~\forall i$
		\STATE Record expert predictions $\cumv^{j}_{t} := \sum_{i=1}^{j}\alpha^{i}_{t}\weakpred^{i}_{t}$
		\STATE Choose an index $i_{t} \in [N]$ 
		\STATE Make a final decision $\predy_{t}=\cumv^{i_{t}}_{t}$
		\STATE Get the true label $\truey_{t}$
		\STATE Compute weights $\alpha^{i}_{t+1},~ \forall i$
		\STATE Compute cost vectors $\cost^{i}_{t},~ \forall i$
		\STATE Weak learners suffer the loss $\cost^{i}_{t} \cdot \weakpred^{i}_{t}$
		\STATE Weak learners update the internal parameters
		\STATE Update booster's parameters, if any
		\ENDFOR
	\end{algorithmic}
	\caption{Online Boosting Schema}
	\label{alg:skeleton}
\end{algorithm}

\section{ALGORITHMS WITH THEORETICAL LOSS BOUNDS}
An essential factor in the performance of boosting algorithms is the predictive power of the individual weak learners. For example, if weak learners make completely random predictions, they cannot produce meaningful outcomes according to the booster's intention. We deal with this matter in two different ways. One way is to define an \textit{edge} of a learner over a completely random learner and assume all weak learners have positive edges. Another way is to measure each learner's \textit{empirical edge} and manipulate the weight $\alpha^{i}_{t}$ to maximize the accuracy of the final prediction. Even a learner that is worse than random guessing can contribute positively if we allow negative weights. The first method leads to OnlineBMR (Section \ref{section:optimal}), and the second to Ada.OLMR (Section \ref{section:adaptive}).

\subsection{Optimal Algorithm}
\label{section:optimal}
We first define the edge of a learner. Recall that weak learners suffer losses determined by cost vectors. Given the true label $\truey$, the booster chooses a cost vector from
\begin{align*}
	\reducedcostseor:=\{\cost \in [0, 1]^{k}~|~&\max_{l\in\truey}\cost[l] \leq \min_{r\notin\truey}\cost[r], \\
	 &\min_{l}\cost[l] = 0 \text{ and } \max_{l}\cost[l] = 1\},
\end{align*}
\noindent where the name $\reducedcostseor$ is used by \cite{jung2017online} and ``eor'' stands for \textit{edge-over-random}. Since the booster wants weak learners to put higher scores at the relevant labels, costs at the relevant labels should be less than those at the irrelevant ones. Restriction to $[0, 1]^{k}$ makes sure that the learner's cost is bounded. Along with cost vectors, the booster passes the importance weights $w_{t} \in [0, 1]$ so that the learner's cost becomes $w_{t} \cost_{t}\cdot \weakpred_{t}$.  

We also construct a \textit{baseline} learner that has edge $\gamma$. Its prediction $\unifv^{\truey}_{\gamma}$ is also a distribution over $[k]$ that puts $\gamma$ more probability for the relevant labels. That is to say, we can write
\begin{align*}
	\unifv^{\truey}_{\gamma}[l] =
	\begin{cases}
	a + \gamma &\text{ if } l \in \truey \\
	a &\text{ if } l \notin \truey, \\
	\end{cases}
\end{align*}
\noindent where the value of $a$ depends on the number of relevant labels, $|\truey|$.

Now we state our online weak learning condition.

\begin{definition}{\bf(OnlineWLC)}
\label{def:onlineWLC}
For parameters $\gamma, \delta \in (0, 1)$, and $S > 0$, a pair of an online learner and an adversary is said to satisfy OnlineWLC $(\delta, \gamma, S)$ if for any $T$, with probability at least $1-\delta$, the learner can generate predictions that satisfy 
\begin{equation*}
\label{eq:onlineWLC}
\sum_{t=1}^{T}w_{t}\cost_{t}\cdot \weakpred_{t} \leq \sum_{t=1}^{T}w_{t}\cost_{t}\cdot\unifv^{\truey_{t}}_{\gamma}+S.
\end{equation*}
$\gamma$ is called an edge, and $S$ an excess loss. 
\end{definition}

This extends the condition made by \citet[Definition 1]{jung2017online}. The probabilistic statement is needed as many online learners produce randomized predictions. The excess loss can be interpreted as a \textit{warm-up period}. Throughout this section, we assume our learners satisfy OnlineWLC $(\delta, \gamma, S)$ with a fixed adversary. 

\paragraph{Cost Vectors} The optimal design of a cost vector depends on the choice of loss. We will use $L^{\truey}(\cumv)$ to denote the loss without specifying it where $\cumv$ is the predicted score vector. The only constraint that we impose on our loss is that it is \textit{proper}, which implies that it is decreasing in $\cumv[l] \text{ for }l \in \truey$, and increasing in $\cumv[r] \text{ for }r \notin \truey$ (readers should note that ``proper loss'' has at least one other meaning in the literature). 

Then we introduce \textit{potential function}, a well known concept in game theory which is first introduced to boosting by \cite{schapire2001drifting}:
\begin{align}
\begin{split}
\label{eq:potential}
	\potential^{0}_{t}(\cumv) &:= L^{\truey_{t}}(\cumv) \\
	\potential^{i}_{t}(\cumv) &:= \E_{l \sim \unifv^{\truey_{t}}_{\gamma}} \potential^{i-1}_{t}(\cumv + \stdv_{l}).
\end{split}
\end{align}
\noindent The potential $\potential^{i}_{t}(\cumv)$ aims to estimate booster's final loss when $i$ more weak learners are left until the final prediction and $\cumv$ is the current state. It can be easily shown by induction that many attributes of $L$ are inherited by potentials. Being proper or convex are good examples. 

Essentially, we want to set 
\begin{equation}
\label{eq:optimalCost}
	\cost^{i}_{t}[l] := \potential^{N-i}_{t}(\cumv^{i-1}_{t} + \stdv_{l}),
\end{equation}
\noindent where $\cumv^{i-1}_{t}$ is the prediction of expert $i-1$. The proper property inherited by potentials ensures the relevant labels have less costs than the irrelevant. To satisfy the boundedness condition of $\reducedcostseor$, we normalize (\ref{eq:optimalCost}) to get 
\begin{equation}
\label{eq:normalizedCost}
	\normcost^{i}_{t}[l] := \frac{\cost^{i}_{t}[l]-\min_{r} \cost^{i}_{t}[r]}{\weightv^{i}[t]},
\end{equation}
\noindent where $\weightv^{i}[t]:=\max_{r} \cost^{i}_{t}[r] - \min_{r} \cost^{i}_{t}[r]$. Since Definition \ref{def:onlineWLC} assumes that $w_{t}\in [0, 1]$, we have to further normalize $\weightv^{i}[t]$. This requires the knowledge of $w^{i*} := \max_{t} \weightv^{i}[t]$. This is unavailable until we observe all the instances, which is fine because we only need this value in proving the loss bound.

\paragraph{Algorithm Details} The algorithm is named by OnlineBMR (Online Boost-by-majority for Multi-label Ranking) as its potential function based design has roots in the classical \textit{boost-by-majority} algorithm (\cite{schapire2001drifting}). In OnlineBMR, we simply set $\alpha^{i}_{t}=1$, or in other words, the booster takes simple cumulative votes. Cost vectors are computed using (\ref{eq:optimalCost}), and the booster always follows the last expert $N$, or $i_{t}=N$. These datails are summarized in Algorithm \ref{alg:OnlineBMR}.

\begin{algorithm}[h]
	\begin{algorithmic}[1]
		\STATE \textbf{Initialize:} $\alpha^{i}_{1}=1$ for $i \in [N]$ 
		\setcounter{ALC@line}{5}
		\STATE Set $i_{t} = N$
		\setcounter{ALC@line}{8}
		\STATE Set the weights $\alpha^{i}_{t+1} =1,~\forall i \in [N]$
		\STATE Set $\cost^{i}_{t}[l] = \potential^{N-i}_{t}(\cumv^{i-1}_{t} + \stdv_{l}),~ \forall l \in [k],~\forall i \in [N]$
		\setcounter{ALC@line}{12}
		\STATE No extra parameters to be updated
	\end{algorithmic}
	\caption{OnlineBMR Details}
	\label{alg:OnlineBMR}
\end{algorithm}

The following theorem holds either if weak learners are single-label learners or if the loss $L$ is convex.

\begin{theorem}{\bf(BMR, General Loss Bound)}
\label{thm:mistakeOptimal}
	For any $T$ and $N \ll \frac{1}{\delta}$, the final loss suffered by OnlineBMR satisfies the following inequality with probability $1 - N\delta$:
	\begin{equation}
		\label{eq:mistakeGeneralOptimal}
		\sum_{t=1}^{T}L^{\truey_{t}}(\predy_{t}) \leq \sum_{t=1}^{T}\potential^{N}_{t}(\textbf{0}) + S \sum_{i=1}^{N}w^{i*}.
	\end{equation}
\end{theorem}

\begin{proof}
From (\ref{eq:potential}) and (\ref{eq:optimalCost}), we can write 
\begin{align*}
	\potential^{N-i+1}_{t}(\cumv^{i-1}_{t}) 
	&= \E_{l \sim \unifv^{Y_{t}}_{\gamma}} \potential^{N-i}_{t}(\cumv^{i-1}_{t} + \stdv_{l}) \\
	&= \cost^{i}_{t} \cdot \unifv^{Y_{t}}_{\gamma} \\
	&= \cost^{i}_{t} \cdot (\unifv^{Y_{t}}_{\gamma} - \weakpred^{i}_{t}) +  \cost^{i}_{t} \cdot \weakpred^{i}_{t}\\
	&\geq \cost^{i}_{t} \cdot (\unifv^{Y_{t}}_{\gamma} - \weakpred^{i}_{t}) + \potential^{N-i}_{t}(\cumv^{i}_{t}),
\end{align*}
\noindent where the last inequality is in fact equality if weak learners are single-label learners, or holds by Jensen's inequality if the loss is convex (which implies the convexity of potentials). Also note that $\cumv^{i}_{t} = \cumv^{i-1}_{t}+\weakpred^{i}_{t}$. Since both $\unifv^{Y_{t}}_{\gamma}$ and $\weakpred^{i}_{t}$ have $\ell_{1}$ norm $1$, we can subtract common numbers from every entry of $\cost^{i}_{t}$ without changing the value of $\cost^{i}_{t} \cdot (\unifv^{Y_{t}}_{\gamma} - \weakpred^{i}_{t})$. This implies we can plug in $\weightv^{i}[t] \normcost^{i}_{t}$ at the place of $\cost^{i}_{t}$. Then we have
\begin{align*}
\begin{split}
	\potential^{N-i+1}_{t}&(\cumv^{i-1}_{t}) -\potential^{N-i}_{t}(\cumv^{i}_{t}) \\
	&\geq \weightv^{i}[t] \normcost^{i}_{t} \cdot \unifv^{Y_{t}}_{\gamma} - \weightv^{i}[t] \normcost^{i}_{t} \cdot \weakpred^{i}_{t}.
\end{split}
\end{align*}
\noindent By summing this over $t$, we have 
\begin{align}
\begin{split}
	\label{eq:optimalProof1}
	\sum_{t=1}^{T}&\potential^{N-i+1}_{t}(\cumv^{i-1}_{t}) - \sum_{t=1}^{T}\potential^{N-i}_{t}(\cumv^{i}_{t})\\
	&\geq \sum_{t=1}^{T}\weightv^{i}[t] \normcost^{i}_{t} \cdot \unifv^{Y_{t}}_{\gamma} - \sum_{t=1}^{T}\weightv^{i}[t] \normcost^{i}_{t} \cdot \weakpred^{i}_{t}.
\end{split}
\end{align}
\noindent OnlineWLC $(\delta, \gamma, S)$ provides, with probability $1-\delta$, 
\begin{align*}
	\sum_{t=1}^{T}\frac{\weightv^{i}[t]}{w^{i*}}\normcost_{t}\cdot \weakpred_{t} 
	&\leq \frac{1}{w^{i*}}\sum_{t=1}^{T}\weightv^{i}[t]\normcost_{t}\cdot\unifv^{\truey_{t}}_{\gamma}+S.
\end{align*}
\noindent Plugging this in (\ref{eq:optimalProof1}), we get
\begin{equation*}
	\sum_{t=1}^{T}\potential^{N-i+1}_{t}(\cumv^{i-1}_{t})-\sum_{t=1}^{T}\potential^{N-i}_{t}(\cumv^{i}_{t})  
	\geq -S w^{i*}.
\end{equation*}
\noindent Now summing this over $i$, we get with probability $1-N\delta$ (due to union bound), 
\begin{equation*}
	\sum_{t=1}^{T}\potential^{N}_{t}(\textbf{0}) + S \sum_{i=1}^{N}w^{i*} \geq \sum_{t=1}^{T}\potential^{0}_{t}(\cumv^{N}_{t}) = \sum_{t=1}^{T}L^{\truey_{t}}(\predy_{t}),
\end{equation*}
which completes the proof.
\end{proof}

Now we evaluate the efficiency of OnlineBMR by fixing a loss. Unfortunately, there is no canonical loss in MLR settings, but following \textit{rank loss} is a strong candidate (cf. \cite{cheng2010bayes} and \cite{gao2011consistency}):
\begin{align*}
\begin{split}
\label{eq:rankLoss}
	\rankLoss^{\truey}&(\cumv) := 
	w_{\truey}\sum_{l\in\truey}\sum_{r\notin\truey} \ind(\cumv[l] < \cumv[r])+\frac{1}{2}\ind(\cumv[l] = \cumv[r]),
\end{split}
\end{align*}
\noindent where $w_{Y}=\frac{1}{|\truey| \cdot |\truey^{c}|}$ is a normalization constant that ensures the loss lies in $[0, 1]$. Note that this loss is not convex. In case weak learners are in fact single-label learners, we can simply use rank loss to compute potentials, but in more general case, we may use the following \textit{hinge loss} to compute potentials:
\begin{equation*}
	\hingeLoss^{\truey}(\cumv) := 
	w_{\truey}\sum_{l\in\truey}\sum_{r\notin\truey} (1+\cumv[r]-\cumv[l])_{+},
\end{equation*}
\noindent where $(\cdot)_{+}:=\max(\cdot, 0)$. It is convex and always greater than rank loss, and thus Theorem \ref{thm:mistakeOptimal} can be used to bound rank loss. In Appendix \ref{section:specificBound}, we bound two terms in the RHS of (\ref{eq:mistakeGeneralOptimal}) when the potentials are built upon rank and hinge losses. Here we record the results. 
\begin{table}[h]
\caption{Upper Bounds for $\potential^{N}_{t}(\textbf{0})$ and $w^{i*}$} \label{table:specificBound}
\begin{center}
\setlength\tabcolsep{12pt}
\begin{tabular}{ccc}
\toprule
loss 			& $\potential^{N}_{t}(\textbf{0})$ 	& $w^{i*}$ \\
\midrule
rank loss		& $e^{-\frac{\gamma^{2}N}{2}}$	&$O(\frac{1}{\sqrt{N-i}} )$	\\
hinge loss		& $(N+1)e^{-\frac{\gamma^{2}N}{2}}$	&2		\\
\bottomrule
\end{tabular}
\end{center}
\end{table}

For the case that we use rank loss, we can check 
\begin{equation*}
	\sum_{i=1}^{N} w^{i*} \leq \sum_{i=1}^{N} O(\frac{1}{\sqrt{N-i}}) \leq O(\sqrt N). 
\end{equation*}
Combining these results with Theorem \ref{thm:mistakeOptimal}, we get the following corollary.

\begin{corollary}{\bf(BMR, Rank Loss Bound)} For any $T$ and $N \ll \frac{1}{\delta}$, OnlineBMR satisfies following rank loss bounds with probability $1-N\delta$. 

With single-label learners, we have 
	\begin{equation}
		\label{eq:mistakeOptimalA}
		\sum_{t=1}^{T}\rankLoss^{\truey_{t}}(\predy_{t}) \leq e^{-\frac{\gamma^{2}N}{2}}T + O(\sqrt N S),
	\end{equation}
and with general learners, we have
	\begin{equation}
		\label{eq:mistakeOptimalB}
		\sum_{t=1}^{T}\rankLoss^{\truey_{t}}(\predy_{t}) \leq (N+1)e^{-\frac{\gamma^{2}N}{2}}T + 2NS.
	\end{equation}
\end{corollary}

\begin{remark}
When we divide both sides by $T$, we find the average loss is asymptotically bounded by the first term. The second term determines the sample complexity. In both cases, the first term decreases exponentially as $N$ grows, which means the algorithm does not require too many learners to achieve a desired loss bound. 
\end{remark}

\paragraph{Matching Lower Bounds} From (\ref{eq:mistakeOptimalA}), we can deduce that to attain average loss less than $\epsilon$, OnlineBMR needs $\Omega(\frac{1}{\gamma^{2}}\ln \frac{1}{\epsilon})$ learners and $\tilde\Omega(\frac{S}{\epsilon\gamma})$ samples. A natural question is whether these numbers are optimal. In fact the following theorem constructs a circumstance that matches these bounds up to logarithmic factors. Throughout the proof, we consider $k$ as a fixed constant.

\begin{theorem}
\label{thm:optimality}
	For any $\gamma \in (0, \frac{1}{2k})$,  $\delta, \epsilon \in (0, 1)$, and $S \geq \frac{k\ln(\frac{1}{\delta})}{\gamma}$, there exists an adversary with a family of learners satisfying OnlineWLC $(\delta, \gamma, S)$ such that to achieve error rate less than $\epsilon$, any boosting algorithm requires at least $\Omega(\frac{1}{\gamma^{2}}\ln \frac{1}{\epsilon})$ learners and $\Omega (\frac{S}{\epsilon\gamma})$ samples.
\end{theorem}

\begin{proof}
	We introduce a sketch here and postpone the complete discussion to Appendix \ref{section:proofs}. We assume that an adversary draws a label $\truey_{t}$ uniformly at random from $2^{[k]}-\{\emptyset, [k]\}$, and the weak learners generate single-label prediction $l_{t}$ w.r.t. $\probv_{t} \in \Delta [k]$. We manipulate $\probv_{t}$ such that weak learners satisfy OnlineWLC $(\delta, \gamma, S)$ but the best possible performance is close to (\ref{eq:mistakeOptimalA}). 
    
	Boundedness conditions in $\reducedcostseor$ and the Azuma-Hoeffding inequality provide that with probability $1-\delta$,     
\begin{equation*}
    	\sum_{t=1}^{T}w_{t}\cost_{t}[l_{t}]
	\leq \sum_{t=1}^{T}w_{t}\cost_{t} \cdot \probv_{t} + \frac{\gamma ||\weightv||_{1}}{k} + \frac{k\ln(\frac{1}{\delta})}{2\gamma}.
\end{equation*}
    For the optimality of the number of learners, we let $\probv_{t}=\unifv^{\truey_{t}}_{2\gamma}$ for all $t$. The above inequality guarantees OnlineWLC is met. Then a similar argument of \citet[Section 13.2.6]{schapire2012boosting} can show that the optimal choice of weights over the learners is $(\frac{1}{N}, \cdots, \frac{1}{N})$. Finally, adopting the argument in the proof of \citet[Theorem 4]{jung2017online}, we can show
\begin{equation*}
	\E \rankLoss^{\truey}(\predy_{t}) \geq \Omega(e^{-4Nk^{2}\gamma^{2}}).
\end{equation*}	
\noindent Setting this value equal to $\epsilon$, we have $N \geq \Omega(\frac{1}{\gamma^{2}}\ln \frac{1}{\epsilon})$, considering $k$ as a fixed constant. This proves the first part of the theorem.

For the second part, let $T_{0}:=\frac{S}{4\gamma}$ and define $\probv_{t} = \unifv^{\truey_{t}}_{0}$ for $t \leq T_{0}$ and $\probv_{t} = \unifv^{\truey_{t}}_{2\gamma}$ for $t > T_{0}$. Then OnlineWLC can be shown to be met in a similar fashion. Observing that weak learners do not provide meaningful information for $t \leq T_{0}$, we can claim any online boosting algorithm suffers a loss at least $\Omega(T_{0})$. Therefore to obtain the certain accuracy $\epsilon$, the number of instances $T$ should be at least $\Omega(\frac{T_{0}}{\epsilon}) = \Omega(\frac{S}{\epsilon \gamma})$, which completes the second part of the proof. 
\end{proof}

\subsection{Adaptive Algorithm}
\label{section:adaptive}
Despite the optimal loss bound, OnlineBMR has a few drawbacks when it is applied in practice. Firstly, potentials do not have a closed form, and their computation becomes a major bottleneck (cf. Table \ref{table:result}). Furthermore, the edge $\gamma$ becomes an extra tuning parameter, which increases the runtime even more. Finally, it is possible that learners have different edges, and assuming a constant edge can lead to inefficiency. To overcome these drawbacks, rather than assuming positive edges for weak learners, our second algorithm chooses the weight $\alpha^{i}_{t}$ adaptively to handle variable edges. 

\paragraph{Surrogate Loss} Like other adaptive boosting algorithms (e.g., \cite{beygelzimer2015optimal} and \cite{freund1999short}), our algorithm needs a surrogate loss. The choice of loss is broadly discussed by \cite{jung2017online}, and logistic loss seems to be a valid choice in online settings as its gradient is uniformly bounded. In this regard, we will use the following \textit{logistic loss}:
\begin{equation*}
\label{eq:logisticLoss}
	\logLoss^{\truey}(\cumv) := w_{\truey}\sum_{l \in \truey} \sum_{r \notin \truey} \log(1 + \exp(\cumv[r] - \cumv[l])).
\end{equation*}
\noindent It is proper and convex. We emphasize that booster's prediction suffers the rank loss, and this surrogate only plays an intermediate role in optimizing parameters. 

\paragraph{Algorithm Details} The algorithm is inspired by \citet[Adaboost.OLM]{jung2017online}, and we call it by Ada.OLMR\footnote{Online, Logistic, Multi-label, and Ranking}. Since it internally aims to minimize the logistic loss, we set the cost vector to be the gradient of the surrogate:
\begin{equation}
\label{eq:adaptiveCost}
	\cost^{i}_{t} := \nabla\logLoss^{\truey_{t}}(\cumv^{i-1}_{t}).
\end{equation}
Next we present how to set the weights $\alpha^{i}_{t}$. Essentially, Ada.OLMR wants to choose $\alpha^{i}_{t}$ to minimize the cumulative logistic loss: 
\begin{equation*}
\sum_{t} \logLoss^{\truey_{t}}(\cumv^{i-1}_{t} + \alpha^{i}_{t} \weakpred^{i}_{t}).
\end{equation*} 
After initializing $\alpha^{i}_{1}$ equals to $0$, we use \textit{online gradient descent} method, proposed by \citet{zinkevich2003online}, to compute the next weights. If we write $f^{i}_{t}(\alpha) := \logLoss^{\truey_{t}}(\cumv^{i-1}_{t}+\alpha \weakpred^{i}_{t})$, we want $\alpha^{i}_{t}$ to satisfy
\begin{equation*}
	\sum_{t} f^{i}_{t} (\alpha^{i}_{t}) \leq \min_{\alpha \in F} \sum_{t}f^{i}_{t}(\alpha) + R^{i}(T),
\end{equation*}
\noindent where $F$ is some \textit{feasible set}, and $R^{i}(T)$ is a sublinear regret. To apply the result by \citet[Theorem 1]{zinkevich2003online}, $f^{i}_{t}$ needs to be convex, and $F$ should be compact. The former condition is met by our choice of logistic loss, and we will use $F = [-2, 2]$ for the feasible set. Since the booster's loss is invariant under the scaling of weights, we can shrink the weights to fit in $F$. 

Taking derivative, we can check ${f^{i}_{t}}'(\alpha) \leq 1$. Now let $\Pi(\cdot)$ denote a projection onto $F$: $\Pi(\cdot) := \max \{-2, \min\{2, \cdot\}\}$. By setting 
\begin{equation*}
	\alpha^i_{t+1} = \Pi (\alpha^i_t - \eta_t {f^i_t}'(\alpha^i_t)) \text{ where } \eta_{t} = \frac{1}{\sqrt t},
\end{equation*}
we get $R^{i}(T) \leq 9 \sqrt T$. Considering that $\cumv^{i}_{t}=\cumv^{i-1}_{t}+\alpha^{i}_{t}\weakpred^{i}_{t}$, we can also write ${f^i_t}'(\alpha^i_t) = \cost^{i+1}_{t}\cdot \weakpred^{i}_{t}$. 

Finally, it remains to address how to choose $i_{t}$. In contrast to OnlineBMR, we cannot show that the last expert is reliably sophisticated. Instead, what can be shown is that at least one of the experts is good enough. Thus we use classical \textit{Hedge algorithm} (cf. \cite{freund1997decision} and \cite{littlestone1989weighted}) to randomly choose an expert at each iteration with adaptive probability distribution depending on each expert's prediction history. In particular, we introduce new variables $v^{i}_{t}$, which are initialized as $v^{i}_{1}=1,~\forall i$. At each iteration, $i_{t}$ is randomly drawn such that 
\begin{equation*}
	\prob(i_{t}=i) \propto v^{i}_{t},
\end{equation*}
\noindent and then $v^{i}_{t}$ is updated based on the expert's rank loss: 
\begin{equation*}
	v^{i}_{t+1}:= v^{i}_{t}e^{-\rankLoss^{\truey_{t}}(\cumv^{i}_{t})}.
\end{equation*}
The details are summarized in Algorithm \ref{alg:Ada.OLMR}.

\begin{algorithm}[h]
	\begin{algorithmic}[1]
		\STATE \textbf{Initialize:} $\alpha^{i}_{1} = 0 \text{ and } v^{i}_{1}=1,~ \forall i \in [N]$
		\setcounter{ALC@line}{5}
		\STATE Randomly draw $i_{t}$ s.t. $\prob(i_{t}=i) \propto v^{i}_{t}$ 
		\setcounter{ALC@line}{8}
		\STATE Compute $\alpha^i_{t+1} = \Pi (\alpha^i_t - \frac{1}{\sqrt t} {f^i_t}'(\alpha^i_t)),~ \forall i \in [N]$
		\STATE Compute $\cost^{i}_{t} = \nabla\logLoss^{\truey_{t}}(\cumv^{i-1}_{t}),~\forall i \in [N]$
		\setcounter{ALC@line}{12}
		\STATE Update $v^{i}_{t+1}= v^{i}_{t}e^{-\rankLoss^{\truey_{t}}(\cumv^{i}_{t})},~\forall i \in [N]$
	\end{algorithmic}
	\caption{Ada.OLMR Details}
	\label{alg:Ada.OLMR}
\end{algorithm}

\paragraph{Empirical Edges} As we are not imposing OnlineWLC, we need another measure of the learner's predictive power to prove the loss bound. From (\ref{eq:adaptiveCost}), it can be observed that the relevant labels have negative costs and the irrelevant ones have positive cost. Furthermore, the summation of entries of $\cost^{i}_{t}$ is exactly $0$. This observation suggests a new definition of weight:
\begin{align}
\begin{split}
\label{eq:adaptiveWeight}
	\weightv^{i}[t] 
	&:= w_{\truey_{t}}\sum_{l \in \truey_{t}} \sum_{r \notin \truey_{t}} \frac{1}{1 + \exp(\cumv^{i-1}_{t}[l] - \cumv^{i-1}_{t}[r])} \\
	&= -\sum_{l\in\truey_{t}}\cost^{i}_{t}[l] 
	= \sum_{r\notin\truey_{t}}\cost^{i}_{t}[r]
	= \frac{||\cost^{i}_{t}||_{1}}{2}.
\end{split}
\end{align}
\noindent This does not directly correspond to the weight used in (\ref{eq:normalizedCost}), but plays a similar role. Then we define the \textit{empirical edge}:
\begin{equation}
\label{eq:empiricalEdge}
	\gamma_{i} := -\frac{\sum_{t=1}^{T}\cost^{i}_{t}\cdot \weakpred^{i}_{t}}{||\weightv^{i}||_{1}}.
\end{equation}
\noindent The baseline learner $\unifv^{\truey_{t}}_{\gamma}$ has this value exactly $\gamma$, which suggests that it is a good proxy for the edge defined in Definition \ref{def:onlineWLC}. 

Now we present the loss bound of Ada.OLMR.

\begin{theorem}{\bf{(Ada.OLMR, Rank loss bound)}}
\label{thm:mistakeAdaptive}
For any $T$ and $N$, with probability $1-\delta$, the rank loss suffered by Ada.OLMR is bounded as follows:
\begin{equation}
\label{eq:mistakeAdaptive}
	\sum_{t=1}^{T} \rankLoss^{\truey_{t}}(\predy_{t}) 
	\leq \frac{8}{\sum_{i}|\gamma_{i}|}T + \tilde O (\frac{N^{2}}{\sum_{i}|\gamma_{i}|}),
\end{equation}
\noindent where $\tilde O$ notation suppresses dependence on $\log \frac{1}{\delta}$.
\end{theorem}

\begin{proof}
We start the proof by defining the rank loss suffered by expert $i$ as below:
\begin{equation*}
	M_{i}:= \sum_{t=1}^{T}\rankLoss^{\truey_{t}}(\cumv^{i}_{t}). 
\end{equation*}
\noindent According to the formula, there is no harm to define $M_{0} = \frac{T}{2}$ since $\cumv^{0}_{t} = \textbf{0}$. As the booster chooses an expert through the Hedge algorithm, a standard analysis (cf. \cite[Corollary 2.3]{cesa2006prediction}) along with the Azuma-Hoeffding inequality provides with probability $1-\delta$, 
\begin{equation}
\label{eq:adaptiveProof1}
	\sum_{t=1}^{T} \rankLoss^{\truey_{t}}(\predy_{t}) \leq 2 \min_{i} M_{i} + 2 \log N + \tilde O (\sqrt T),
\end{equation}
\noindent where $\tilde O$ notation suppresses dependence on $\log \frac{1}{\delta}$.

It is not hard to check that $\frac{1}{1+\exp(a-b)} \geq \frac{1}{2}\ind(a\leq b)$, from which we can infer 
\begin{equation}
\label{eq:adaptiveProof8}
	\weightv^{i}[t] \geq \frac{1}{2}\rankLoss^{\truey_{t}}(\cumv^{i-1}_{t}) \text{ and }
	||\weightv^{i}||_{1} \geq \frac{M_{i-1}}{2},
\end{equation}
\noindent where $\weightv^{i}$ is defined in (\ref{eq:adaptiveWeight}). Note that this relation holds for the case $i=1$ as well. 

Now let $\Delta_{i}$ denote the difference of the cumulative logistic loss between two consecutive experts: 
\begin{align*}
	\Delta_{i} 
	&:= \sum_{t=1}^{T} \logLoss^{\truey_{t}}(\cumv^{i}_{t}) - \logLoss^{\truey_{t}}(\cumv^{i-1}_{t}) \\
	&= \sum_{t=1}^{T} \logLoss^{\truey_{t}}(\cumv^{i-1}_{t} + \alpha^{i}_{t}\weakpred^{i}_{t}) - \logLoss^{\truey_{t}}(\cumv^{i-1}_{t}). 
\end{align*}
Then the online gradient descent algorithm provides 
\begin{align}
\begin{split}
\label{eq:adaptiveProof2}
	\Delta_{i} 
	\leq \min_{\alpha \in [-2, 2]} \sum_{t=1}^{T} [\logLoss^{\truey_{t}}(\cumv^{i-1}_{t} + \alpha&\weakpred^{i}_{t}) - \logLoss^{\truey_{t}}(\cumv^{i-1}_{t})]  \\
	&+ 9 \sqrt T. 
\end{split}
\end{align}
\noindent Here we record an univariate inequality:
\begin{align*}
	\log(1+e^{s+\alpha}) - \log(1+e^s) 
	&= \log(1 + \frac{e^{\alpha}-1}{1+ e^{-s}}) \\
	&\leq \frac{1}{1 + e^{-s}} (e^\alpha - 1).
\end{align*}
\noindent We expand the difference to get 
\begin{align}
\begin{split}
\label{eq:adaptiveProof4}
	&\sum_{t=1}^{T}[\logLoss^{\truey_{t}}(\cumv^{i-1}_{t} + \alpha\weakpred^{i}_{t}) - \logLoss^{\truey_{t}}(\cumv^{i-1}_{t})]\\
	&= \sum_{t=1}^{T}\sum_{l\in\truey_{t}}\sum_{r\notin\truey_{t}}\log\frac{1+e^{\cumv^{i-1}_{t}[r] - \cumv^{i-1}_{t}[l] + \alpha(\weakpred^{i}_{t}[r] - \weakpred^{i}_{t}[l])}}{1+e^{\cumv^{i-1}_{t}[r] - \cumv^{i-1}_{t}[l]}}\\
	&\leq \sum_{t=1}^{T}\sum_{l\in\truey_{t}}\sum_{r\notin\truey_{t}}\frac{1}{1+e^{\cumv^{i-1}_{t}[l] - \cumv^{i-1}_{t}[r]}}(e^{\alpha(\weakpred^{i}_{t}[r] - \weakpred^{i}_{t}[l])}-1) \\
	&=: f(\alpha).
\end{split}
\end{align}
\noindent We claim that $\min_{\alpha \in [-2, 2]}f(\alpha) \leq -\frac{|\gamma_{i}|}{2}||\weightv^{i}||_{1}$. Let us rewrite $||\weightv^{i}||_{1}$ in (\ref{eq:adaptiveWeight}) and $\gamma_{i}$ in (\ref{eq:empiricalEdge}) as following. 
\begin{align}
\begin{split}
\label{eq:adaptiveProof5}
	||\weightv^{i}||_{1}
	&=\sum_{t=1}^{T}\sum_{l\in\truey_{t}}\sum_{r\notin\truey_{t}}\frac{1}{1+e^{\cumv^{i-1}_{t}[l] - \cumv^{i-1}_{t}[r]}}\\
	\gamma_{i} 
	&= \sum_{t=1}^{T}\sum_{l\in\truey_{t}}\sum_{r\notin\truey_{t}}\frac{1}{||\weightv^{i}||_{1}}\frac{\weakpred^{i}_{t}[l] - \weakpred^{i}_{t}[r]}{1+e^{\cumv^{i-1}_{t}[l] - \cumv^{i-1}_{t}[r]}}. 
\end{split}
\end{align}
\noindent For the ease of notation, let $j$ denote an index that moves through all tuples of $(t, l, r) \in [T] \times \truey_{t} \times \truey_{t}^{c}$, and $a_{j}$ and $b_{j}$ denote following terms. 
\begin{align*}
	a_{j} &= \frac{1}{||\weightv^{i}||_{1}}\frac{1}{1+e^{\cumv^{i-1}_{t}[l] - \cumv^{i-1}_{t}[r]}} \\
	b_{j} &= \weakpred^{i}_{t}[l] - \weakpred^{i}_{t}[r].
\end{align*}
\noindent Then from (\ref{eq:adaptiveProof5}), we have $\sum_{j}a_{j} = 1$ and $\sum_{j}a_{j}b_{j} = \gamma_{i}$. Now we express $f(\alpha)$ in terms of $a_{j}$ and $b_{j}$ as below. 
\begin{equation*}
	\frac{f(\alpha)}{||\weightv^{i}||_{1}}
	= \sum_{j}a_{j}(e^{-\alpha b_{j}}-1) 
	\leq e^{-\alpha\sum_{j}a_{j}b_{j}}-1 
	= e^{-\alpha \gamma_{i}}-1,
\end{equation*}
\noindent where the inequality holds by Jensen's inequality. From this, we can deduce that 
\begin{equation*}
\min_{\alpha\in [-2, 2]}\frac{f(\alpha)}{||\weightv^{i}||_{1}} \leq e^{-2|\gamma_{i}|}-1 \leq -\frac{|\gamma_{i}|}{2},
\end{equation*}
\noindent where the last inequality can be checked by investigating $|\gamma_{i}| = 0, 1$ and observing the convexity of the exponential function. This proves our claim that 
\begin{equation}
\label{eq:adaptiveProof7}
	\min_{\alpha \in [-2, 2]}f(\alpha) \leq -\frac{|\gamma_{i}|}{2}||\weightv^{i}||_{1}.
\end{equation} 
\noindent Combining (\ref{eq:adaptiveProof8}), (\ref{eq:adaptiveProof2}), (\ref{eq:adaptiveProof4}) and (\ref{eq:adaptiveProof7}), we have
\begin{equation*}
\Delta_{i} \leq -\frac{|\gamma_{i}|}{4} M_{i-1} + 9 \sqrt T.
\end{equation*}
\noindent Summing over $i$, we get by telescoping rule
\begin{align*}
	\sum_{t=1}^{T} \logLoss^{\truey_{t}}&(\cumv^{N}_{t}) - \sum_{t=1}^{T} \logLoss^{\truey_{t}}(\zerov) \\
	&\leq -\frac{1}{4}\sum_{i=1}^{N} |\gamma_{i}| M_{i-1} + 9N\sqrt T \\
	&\leq -\frac{1}{4}\sum_{i=1}^{N} |\gamma_{i}| \min_{i} M_{i} + 9N \sqrt T.
\end{align*}
\noindent Note that $\logLoss^{\truey_{t}}(\zerov) = \log 2$ and $\logLoss^{\truey_{t}}(\cumv^{N}_{t}) \geq 0$. Therefore we have
\begin{equation*}
	\min_{i}M_{i} \leq  \frac{4\log 2}{\sum_{i}|\gamma_{i}|}T + \frac{36N \sqrt T}{\sum_{i}|\gamma_{i}|}.
\end{equation*}
\noindent Plugging this in (\ref{eq:adaptiveProof1}), we get with probability $1-\delta$,
\begin{align*}
	\sum_{t=1}^{T} \rankLoss^{\truey_{t}}(\predy_{t})
	&\leq \frac{8 \log 2}{\sum_{i}|\gamma_{i}|}T + \tilde O (\frac{N \sqrt T}{\sum_{i}|\gamma_{i}|} + \log N) \\
	&\leq \frac{8}{\sum_{i}|\gamma_{i}|}T + \tilde O (\frac{N^{2}}{\sum_{i}|\gamma_{i}|}),
\end{align*}
where the last inequality holds from AM-GM inequality: $cN\sqrt T \leq \frac{c^{2}N^{2} + T}{2}$. This completes our proof.
\end{proof}

\paragraph{Comparison with OnlineBMR}
We finish this section by comparing our two algorithms. For a fair comparison, assume that all learners have edge $\gamma$. Since the baseline learner $\unifv^{\truey}_{\gamma}$ has empirical edge $\gamma$, for sufficiently large $T$, we can deduce that $\gamma_{i}\geq \gamma$ with high probability. Using this relation, (\ref{eq:mistakeAdaptive}) can be written as
\begin{equation*}
	\sum_{t=1}^{T} \rankLoss^{\truey_{t}}(\predy_{t}) 
	\leq \frac{8}{N\gamma}T + \tilde O (\frac{N}{\gamma}).
\end{equation*}
Comparing this to either (\ref{eq:mistakeOptimalA}) or (\ref{eq:mistakeOptimalB}), we can see that OnlineBMR indeed has better asymptotic loss bound and sample complexity. Despite this sub-optimality (in upper bounds), Ada.OLMR shows comparable results in real data sets due to its adaptive nature.

\section{EXPERIMENTS}
\label{section:experiments}
We performed an experiment on benchmark data sets taken from MULAN\footnote{\cite{mulan}, \url{http://mulan.sourceforge.net/datasets.html}}. We chose these four particular data sets because \cite{dembczynski2012consistent} already provided performances of batch setting boosting algorithms, giving us a benchmark to compare with. The authors in fact used five data sets, but \textit{image} data set is no longer available from the source. Table \ref{table:data} summarizes the basic statistics of data sets, including training and test set sizes, number of features and labels, and three statistics of the sizes of relevant sets. The data set \textit{m-reduced} is a reduced version of \textit{mediamill} obtained by random sampling without replacement. We keep the original split for training and test sets to provide more relevant comparisons.

\begin{table}[h]
\caption{Summary of Data Sets} \label{table:data}
\begin{center}
\setlength\tabcolsep{2pt}
\begin{tabular}{lrrrrrrr}
\toprule
data & \#train & \#test & dim & $k$ & min &mean & max \\
\midrule 
emotions 		&391		&202		&72		&6		&1		&1.87		&3\\
scene 		&1211	&1196	&294		&6		&1		&1.07		&3\\
yeast 		&1500	&917		&103		&14		&1		&4.24		&11\\
mediamill 		&30993	&12914	&120		&101		&0		&4.38		&18\\
m-reduced 	&1500	&500		&120		&101		&0		&4.39		&13\\
\bottomrule
\end{tabular}
\end{center}
\end{table}

VFDT algorithms presented by \cite{domingos2000mining} were used as weak learners. Every algorithm used $100$ trees whose parameters were randomly chosen. VFDT is trained using single-label data, and we fed individual relevant labels along with importance weights that were computed as $\max_{l}\cost_{t}-\cost_{t}[l]$. Instead of using all covariates, the booster fed to trees randomly chosen $20$ covariates to make weak predictions less correlated. 

All computations were carried out on a Nehalem architecture 10-core 2.27 GHz Intel Xeon E7-4860 processors with 25 GB RAM per core. Each algorithm was trained at least ten times\footnote{OnlineBMR for \textit{m-reduced} was tested 10 times due to long runtimes, and others were tested 20 times} with different random seeds, and the results were aggregated through mean. Predictions were evaluated by rank loss. The algorithm's loss was only recorded for test sets, but it kept updating its parameters while exploring test sets as well. 

Since VFDT outputs a conditional distribution, which is not of a single-label format, we used hinge loss to compute potentials. Furthermore, OnlineBMR has an additional parameter of edge $\gamma$. We tried four different values\footnote{$\{.2, .1, .01, .001\}$ for small $k$ and $\{.05, .01, .005, .001\}$ for large $k$}, and the best result is recorded as \textit{best BMR}. Table \ref{table:result} summarizes the results. 

\begin{table}[h]
\caption{Average Loss and Runtime in seconds} \label{table:result}
\begin{center}
\setlength\tabcolsep{3pt}
\begin{tabular}{lrrrrrrrr}
\toprule
data	&&batch\footnote{The best result from batch boosting algorithms in \cite{dembczynski2012consistent}}	&&\multicolumn{2}{c}{Ada.OLMR}&&	\multicolumn{2}{c}{best BMR}\\
\midrule
emotions		&&.1699	&&.1600	&253		&&.1654	&611		\\
scene		&&.0720	&&.0881	&341		&&.0743	&1488	\\
yeast		&&.1820	&&.1874	&2675	&&.1836	&9170	\\
mediamill		&&.0665	&&.0508	&69565	&&-		&-		\\
m-reduced	&&-		&&.0632	&4148	&&.0630	&288204	\\
\bottomrule
\end{tabular}
\end{center}
\end{table}

Two algorithms' average losses are comparable to each other and to batch setting results, but OnlineBMR requires much longer runtimes. Based on the fact that best BMR's performance is reported on the best edge parameter out of four trials, Ada.OLMR is far more favorable in practice. With large number of labels, runtime for OnlineBMR grows rapidly, and it was even impossible to run \textit{mediamill} data within a week, and this was why we produced the reduced version. The main bottleneck is the computation of potentials as they do not have closed form.

\section{CONCLUSION}
In this paper, we presented two online boosting algorithms that make multi-label ranking (MLR) predictions. The algorithms are quite flexible in their choice of weak learners in that various types of learners can be combined to produce a strong learner. OnlineBMR is built upon the assumption that all weak learners are strictly better than random guessing, and its loss bound is shown to be tight under certain conditions. Ada.OLMR adaptively chooses the weights over the learners so that learners with arbitrary (even negative) edges can be boosted. Despite its suboptimal loss bound, it produces comparable results with OnlineBMR and runs much faster. 

Online MLR boosting provides several opportunities for further research. A major issue in MLR problems is that there does not exist a canonical loss. Fortunately, Theorem \ref{thm:mistakeOptimal} holds for any proper loss, but Ada.OLMR only has a rank loss bound. An adaptive algorithm that can handle more general losses will be desirable. The existence of an \emph{optimal} adaptive algorithm is another interesting open question. 

\subsubsection*{Acknowledgments}
We acknowledge the support of NSF via CAREER grant IIS-1452099 and CIF-1422157 and of the Sloan Foundation via a Sloan Research Fellowship.

\bibliography{./tex/myref} 
\clearpage
\newpage

\titleformat{\section}[hang]
{\large\bfseries}
{\appendixname~\thesection.}{0.5em}{}

\begin{appendices}

\section{SPECIFIC BOUNDS FOR OnlineBMR}
\label{section:specificBound}

We begin this section by introducing a \textit{random walk framework} to compute potentials. Suppose $\Xv^{i}:=(X_{1}, \cdots, X_{k})$ is a random vector that tracks the number of draws of each label among $i$ i.i.d. random draws w.r.t. $\unifv^{\truey_{t}}_{\gamma}$. Then according to (\ref{eq:potential}), we may write
\begin{equation*}
	\potential^{i}_{t}(\cumv) = \E L^{\truey_{t}}(\cumv + \Xv).
\end{equation*}
This framework will appear frequently throughout the proofs. We start from rank loss.

\begin{lemma} 
\label{lemma:asymptoticErrorA}
Under the same setting as in Theorem \ref{thm:mistakeOptimal} but with potentials built upon rank loss, we may bound $\potential^{N}_{t}(\textbf{0})$ as following:
\begin{equation*}
	\potential^{N}_{t}(\textbf{0}) \leq e^{-\frac{\gamma^{2}N}{2}} .
\end{equation*}
\end{lemma}

\begin{proof}
For simplicity, we drop $t$ in the proof. Let $\Xv^{N}$ be the aforementioned random vector. Then we may write the potential by
\begin{align*}
	\potential^{N}(\textbf{0}) 
	&= \E \rankLoss^{\truey}(\Xv^{N}) \\
	&\leq w_{\truey}\sum_{l\in\truey}\sum_{r\notin\truey} \E\ind(X_{r} \geq X_{l}) \\
	&= w_{\truey}\sum_{l\in\truey}\sum_{r\notin\truey}\prob(X_{r}-X_{l} \geq 0). \\
\end{align*}
\noindent Fix $l\in\truey$ and $r \notin \truey$. By definition of $\unifv^{\truey}_{\gamma}$, we have 
\begin{equation*}
	a:=\unifv^{\truey}_{\gamma}[l] = \unifv^{\truey}_{\gamma}[r] + \gamma=:b.
\end{equation*}
Now suppose we draw $1$ with probability $a$, $-1$ with probability $b$, and $0$ otherwise. Then $\prob(X_{r}-X_{l} \geq 0)$ equals the probability that the summation of $N$ i.i.d. random numbers is non-negative. Then we can apply the Hoeffding's inequality to get 
\begin{equation*}
	\prob(X_{r}-X_{l} \geq 0) \leq e^{-\frac{\gamma^{2}N}{2}}. 
\end{equation*}
\noindent Since $w_{\truey}$ is the inverse of the number of pairs $(l, r)$, this proves our assertion. 
\end{proof}

\begin{lemma}
\label{lemma:weightA}

Under the same setting as in Theorem \ref{thm:mistakeOptimal} but with potentials built upon rank loss, we can show that $\forall i,~w^{i*} \leq O(\frac{1}{\sqrt{N-i}})$. 
\end{lemma}

\begin{proof}
First we fix $t$ and $i$. We also fix $l^{*}\in \truey_{t}$ and $r^{*}\in \truey_{t}^{c}$. Then write $\cumv_{1}:= \cumv^{i-1}_{t}+\stdv_{l^{*}}$ and $\cumv_{2}:= \cumv^{i-1}_{t}+\stdv_{r^{*}}$. Again we introduce $\Xv^{N-i}$. Then we may write
\begin{align*}
	\cost^{i}_{t}[r^{*}] - \cost^{i}_{t}[l^{*}]
	&=\potential^{N-i}_{t}(\cumv_{2}	) - \potential^{N-i}_{t}(\cumv_{1})\\
	&=\E[\rankLoss^{\truey_{t}}(\cumv_{2}+\Xv^{N-i}) - \rankLoss^{\truey_{t}}(\cumv_{1}+\Xv^{N-i})]\\
	&\leq w_{\truey_{t}}\sum_{l\in\truey_{t}}\sum_{r\notin\truey_{t}}f(r, l),
\end{align*}
\noindent where 
\begin{align*}
	f(r,l) := \E[\ind&(\cumv_{2}[r] + X_{r} \geq \cumv_{2}[l] + X_{l}) \\
	&-\ind(\cumv_{1}[r] + X_{r} > \cumv_{1}[l] + X_{l})].
\end{align*}
\noindent Here we intentionally include and exclude equality for the ease of computation. Changing the order of terms, we can derive
\begin{align*}
	f(r,l) 
	&\leq \prob(\cumv_{1}[l]- \cumv_{1}[r]\geq X_{r}-X_{l} \geq \cumv_{2}[l]- \cumv_{2}[r]) \\
	&\leq 3 \max_{n} \prob(X_{r}-X_{l} = n),
\end{align*}
\noindent where the last inequality is deduced from the fact that 
\begin{equation*}
	(\cumv_{1}[l]- \cumv_{1}[r]) - (\cumv_{2}[l]- \cumv_{2}[r]) \in \{0, 1, 2\}.
\end{equation*}
Using Berry-Esseen theorem, it is shown by \citet[Lemma 10]{jung2017online} that $\max_{n} \prob(X_{r}-X_{l} = n) \leq O(\frac{1}{\sqrt{N-i}})$, which implies that 
\begin{equation*}
\cost^{i}_{t}[r^{*}] - \cost^{i}_{t}[l^{*}] \leq O(\frac{1}{\sqrt{N-i}}). 
\end{equation*}
Since $l^{*}$ and $r^{*}$ are arbitrary, and the bound does not depend on $t$, the last inequality proves our assertion. 
\end{proof}

Now we provide similar bounds when the potentials are computed from hinge loss.

\begin{lemma}
\label{lemma:asymptoticErrorB}

Under the same setting as in Theorem \ref{thm:mistakeOptimal} but with potentials built upon hinge loss, we may bound $\potential^{N}_{t}(\textbf{0})$ as following:
\begin{equation*}
	\potential^{N}_{t}(\textbf{0}) \leq (N+1)e^{-\frac{\gamma^{2}N}{2}} .
\end{equation*}
\end{lemma}

\begin{proof}
Again we drop $t$ in the proof and introduce $\Xv^{N}$. Then we may write the potential by
\begin{align*}
	\potential^{N}(\textbf{0}) 
	&= \E \hingeLoss^{\truey}(\Xv^{N}) \\
	&= w_{\truey}\sum_{l\in\truey}\sum_{r\notin\truey} \E (1+X_{r}-X_{l})_{+} \\
	&= w_{\truey}\sum_{l\in\truey}\sum_{r\notin\truey} \sum_{n=0}^{N}\prob(X_{r}-X_{l} \geq n) \\
	&\leq w_{\truey}\sum_{l\in\truey}\sum_{r\notin\truey} (N+1)\prob(X_{r}-X_{l} \geq 0).
\end{align*}
We already checked in Lemma \ref{lemma:asymptoticErrorA} that
\begin{equation*}
	\prob(X_{r}-X_{l} \geq 0) \leq e^{-\frac{\gamma^{2}N}{2}},
\end{equation*}
\noindent which concludes the proof. 
\end{proof}

\begin{lemma}
\label{lemma:weightB}

Under the same setting as in Theorem \ref{thm:mistakeOptimal} but with potentials built upon hinge loss, we can show that $\forall i,~w^{i*} \leq 2$.
\end{lemma}

\begin{proof}
First we fix $t$ and $i$. We also fix $l^{*}\in \truey_{t}$ and $r^{*}\in \truey_{t}^{c}$. Then write $\cumv_{1}:= \cumv^{i-1}_{t}+\stdv_{l*}$ and $\cumv_{2}:= \cumv^{i-1}_{t}+\stdv_{r*}$. Again with $\Xv^{N-i}$, we may write
\begin{align*}
	\cost^{i}_{t}[r^{*}] &- \cost^{i}_{t}[l^{*}] = \potential^{N-i}_{t}(\cumv_{2}	) - \potential^{N-i}_{t}(\cumv_{1})\\
	&=\E[\hingeLoss^{\truey_{t}}(\cumv_{2}+\Xv^{N-i}) - \hingeLoss^{\truey_{t}}(\cumv_{1}+\Xv^{N-i})]\\
	&=w_{\truey_{t}}\sum_{l\in\truey_{t}}\sum_{r\notin\truey_{t}}f(r, l),
\end{align*}
\noindent where 
\begin{align*}
	f(r,l) &:= \E[(1+(\cumv_{2}+\Xv^{N-i})[r] - (\cumv_{2}+\Xv^{N-i})[l])_{+} \\
	&-(1+(\cumv_{1}+\Xv^{N-i})[r] - (\cumv_{1}+\Xv^{N-i})[l])_{+}].
\end{align*}
\noindent It is not hard to check that the term inside the expectation is always bounded above by $2$. This fact along with the definition of $w_{\truey_{t}}$ provides that $\cost^{i}_{t}[r^{*}] - \cost^{i}_{t}[l^{*}] \leq 2$. Since our choice of $l^{*}$ and $r^{*}$ are arbitrary, this proves $\weightv^{i}[t] \leq 2$, which completes the proof. 
\end{proof}

\section{COMPLETE PROOF OF THEOREM \ref{thm:optimality}}
\label{section:proofs}

\begin{proof}
   	We assume that an adversary draws a label $\truey_{t}$ uniformly at random from $2^{[k]}-\{\emptyset, [k]\}$, and the weak learners generate single-label predictions w.r.t. $\probv_{t} \in \Delta [k]$. Any boosting algorithm can only make a final decision by weighted cumulative votes of $N$ weak learners. We manipulate $\probv_{t}$ such that weak learners satisfy OnlineWLC $(\delta, \gamma, S)$ but the best possible performance is close to (\ref{eq:mistakeOptimalA}). 
    
	As we are assuming single-label predictions, $\weakpred_{t}=\stdv_{l_{t}}$ for some $l_{t} \in [k]$ and $\cost_{t} \cdot \weakpred_{t} = \cost_{t}[l_{t}]$. Furthermore, the bounded condition of $\reducedcostseor$ ensures $\cost_{t}[l_{t}]$ is contained in $[0, 1]$. The Azuma-Hoeffding inequality provides that with probability $1-\delta$,     
\begin{align}
    	\begin{split}
    	\label{eq:optimality1}
    	\sum_{t=1}^{T}w_{t}\cost_{t}[l_{t}]
	&\leq \sum_{t=1}^{T}w_{t}\cost_{t} \cdot \probv_{t}  + \sqrt {2||\weightv||^{2}_{2} \ln(\frac{1}{\delta})} \\
	&\leq \sum_{t=1}^{T}w_{t}\cost_{t} \cdot \probv_{t} + \frac{\gamma ||\weightv||^{2}_{2}}{k} + \frac{k\ln(\frac{1}{\delta})}{2\gamma}\\
	&\leq \sum_{t=1}^{T}w_{t}\cost_{t} \cdot \probv_{t} + \frac{\gamma ||\weightv||_{1}}{k} + \frac{k\ln(\frac{1}{\delta})}{2\gamma},
	\end{split}
\end{align}
\noindent where the second inequality holds by arithmetic mean and geometric mean relation and the last inequality holds due to $w_{t} \in [0, 1]$.
    
    We start from providing a lower bound on the number of weak learners. Let $\probv_{t}=\unifv^{\truey_{t}}_{2\gamma}$ for all $t$. This can be done by the constraint $\gamma < \frac{1}{4k}$. From the condition of $\reducedcostseor$ that $\min_{l}\cost[l] = 0, \max_{l}\cost=1$ along with the fact that $\truey \notin \{\emptyset, [k]\}$, we can show that $\cost\cdot(\unifv^{\truey}_{\gamma}-\unifv^{\truey}_{2\gamma}) \geq \frac{\gamma}{k}$. Then the last line of (\ref{eq:optimality1}) becomes 
\begin{align*}
    	\sum_{t=1}^{T}w_{t}&\cost_{t} \cdot \unifv^{\truey_{t}}_{2\gamma} + \frac{\gamma ||\weightv||_{1}}{k} + \frac{k\ln(\frac{1}{\delta})}{2\gamma} \\
	&\leq \sum_{t=1}^{T}(w_{t}\cost_{t}\cdot \unifv^{\truey_{t}}_{\gamma} - \frac{\gamma w_{t}}{k}) +\frac{\gamma ||\weightv||_{1}}{k} + \frac{k\ln(\frac{1}{\delta})}{2\gamma} \\
	&\leq \sum_{t=1}^{T}w_{t}\cost_{t}\cdot \unifv^{\truey_{t}}_{\gamma}+ S,
\end{align*}
\noindent which validates that weak learners indeed satisfy OnlineWLC $(\delta, \gamma, S)$. Following the argument of \citet[Section 13.2.6]{schapire2012boosting}, we can also prove that the optimal choice of weights over the learners is $(\frac{1}{N}, \cdots, \frac{1}{N})$. 

	Now we compute a lower bound for the booster's loss. Let $\Xv := (X_{1}, \cdots, X_{k})$ be a random vector that tracks the number of labels drawn from $N$ i.i.d. random draws w.r.t. $\unifv^{\truey}_{2\gamma}$. Then the expected rank loss of the booster can be written as:
\begin{align*}
\E \rankLoss^{\truey}(\Xv) 
&\geq w_{\truey}\sum_{l\in\truey}\sum_{r\notin\truey}\prob(X_{l}< X_{r}).
\end{align*}
\noindent Adopting the arguments in the proof by \citet[Theorem 4]{jung2017online}, we can show that 
\begin{equation*}
\prob(X_{l}< X_{r}) \geq \Omega(e^{-4Nk^{2}\gamma^{2}}). 
\end{equation*}
\noindent This shows $\E \rankLoss^{\truey}(\Xv) \geq \Omega(e^{-4Nk^{2}\gamma^{2}})$. Setting this value equal to $\epsilon$, we have $N \geq \Omega(\frac{1}{\gamma^{2}}\ln \frac{1}{\epsilon})$, considering $k$ as a fixed constant. This proves the first part of the theorem.

Now we move on to the optimality of sample complexity. We record another inequality that can be checked from the conditions of $\reducedcostseor$: $\cost\cdot(\unifv^{\truey}_{0} - \unifv^{\truey}_{\gamma}) \leq \gamma$. Let $T_{0}:=\frac{S}{4\gamma}$ and define $\probv_{t} = \unifv^{\truey_{t}}_{0}$ for $t \leq T_{0}$ and $\probv_{t} = \unifv^{\truey_{t}}_{2\gamma}$ for $t > T_{0}$. Then for $T \leq T_{0}$, (\ref{eq:optimality1}) implies 
\begin{align}
\begin{split}
\label{eq:optimality2}
    	\sum_{t=1}^{T}&w_{t}\cost_{t}[l_{t}] \\
	&\leq \sum_{t=1}^{T}w_{t}\cost_{t}\cdot\unifv^{\truey_{t}}_{0} + \frac{\gamma||\weightv||_{1}}{k} + \frac{k \ln (\frac{1}{\delta})}{2\gamma} \\
	&\leq \sum_{t=1}^{T}w_{t}\cost_{t}\cdot\unifv^{\truey_{t}}_{\gamma} + \gamma(1+\frac{1}{k})||\weightv||_{1} + \frac{k \ln (\frac{1}{\delta})}{2\gamma} \\
	&\leq \sum_{t=1}^{T}w_{t}\cost_{t}\cdot\unifv^{\truey_{t}}_{\gamma} + S. 
\end{split}
\end{align}
    where the last inequality holds because $||\weightv||_{1} \leq T_{0} = \frac{S}{4\gamma}$. For $ T > T_{0}$, again (\ref{eq:optimality1}) implies 
    \begin{align}
    	\begin{split}
    	\label{eq:optimality3}
    	\sum_{t=1}^{T}&w_{t}\cost_{t}[l_{t}] \leq \sum_{t=1}^{T_{0}}w_{t}\cost_{t}\cdot\unifv^{\truey_{t}}_{0} +\sum_{t=T_{0}+1}^{T}w_{t}\cost_{t}\cdot\unifv^{\truey_{t}}_{2\gamma} \\
	&~~~~~~~~~~~~~~~+ \frac{\gamma ||\weightv||_{1}}{k} + \frac{k \ln(\frac{1}{\delta})}{2\gamma} \\
    	&\leq \sum_{t=1}^{T}w_{t}\cost_{t}\cdot\unifv^{\truey_{t}}_{\gamma} + \frac{k+1}{k}\gamma T_{0}+ \frac{k \ln(\frac{1}{\delta})}{2\gamma}\\
	&\leq \sum_{t=1}^{T}w_{t}\cost_{t}\cdot\unifv^{\truey_{t}}_{\gamma} + S.
    	\end{split}
    \end{align}
    (\ref{eq:optimality2}) and (\ref{eq:optimality3}) prove that the weak learners indeed satisfy OnlineWLC $(\delta, \gamma, S)$. Observing that weak learners do not provide meaningful information for $t \leq T_{0}$, we can claim any online boosting algorithm suffers a loss at least $\Omega(T_{0})$. Therefore to get the certain accuracy, the number of instances $T$ should be at least $\Omega(\frac{T_{0}}{\epsilon}) = \Omega(\frac{S}{\epsilon \gamma})$, which completes the second part of the proof. 
\end{proof}

\end{appendices}

\end{document}